\documentclass{article}
\pdfoutput=1
\usepackage{iclr2018_conference,times}
\usepackage{amsmath}
\usepackage{amssymb}
\usepackage{amsthm}
\usepackage{hyperref}
\usepackage{xcolor}
\hypersetup{
    colorlinks,
    linkcolor={blue!50!black},
    citecolor={blue!50!black},
    urlcolor={blue!50!black}
}
\usepackage[pdftex]{graphicx}
\usepackage{url}
\usepackage{booktabs}
\usepackage{siunitx}
\usepackage{multirow}
\usepackage{appendix}
\usepackage{enumerate}
\usepackage{bbold}

\DeclareSIUnit[number-unit-product={}]{\percent}{\%}

\newtheorem{proposition}{Proposition}
\newtheorem{lemma}{Lemma}
\newtheorem{definition}{Definition}

\newcommand{\figscale}{1.0}

\iclrfinalcopy

\title{Trace norm regularization and faster inference for 
embedded speech recognition RNNs}

\author{
Markus Kliegl, Siddharth Goyal, Kexin Zhao, Kavya Srinet \& Mohammad Shoeybi  \\
Baidu Silicon Valley Artificial Intelligence Lab \\
\texttt{\{klieglmarkus,goyalsiddharth,zhaokexin01,srinetkavya,}
\\\texttt{ mohammad\}@baidu.com}
}

\begin{document}
\maketitle

\begin{abstract}
We propose and evaluate new techniques for compressing and speeding up dense
matrix multiplications as found in the fully connected and recurrent layers of
neural networks for embedded large vocabulary continuous speech recognition
(LVCSR). For compression, we introduce and study a trace norm
regularization technique for training low rank factored versions of matrix
multiplications. Compared to standard low rank training, we show that our method
leads to good accuracy versus number of parameter trade-offs
and can be used to speed up training of large models. For speedup, we
enable faster inference on ARM processors through new open sourced kernels
optimized for small batch sizes, resulting in 3x to 7x speed ups over the widely
used gemmlowp library. Beyond LVCSR, we expect our techniques and kernels to be
more generally applicable to embedded neural networks with large fully connected
or recurrent layers.
\end{abstract}

\section{Introduction}
\label{sec:intro}

For embedded applications of machine learning, we seek models that are
as accurate as possible given constraints on size and on latency at inference time. For
many neural networks, the parameters and computation are concentrated in two
basic building blocks:
\begin{enumerate}
    \item \textbf{Convolutions}. These tend to dominate in, for example, image
    processing applications.
    \item \textbf{Dense matrix multiplications} (GEMMs) as found, for example,
    inside fully connected layers or recurrent layers such as GRU and LSTM. These
    are common in speech and natural language processing applications.
\end{enumerate}

These two building blocks are the natural targets for efforts to reduce
parameters and speed up models for embedded applications. Much work on this topic
already exists in the literature. For a brief overview, see
Section~\ref{sec:related}.

In this paper, we focus only on dense matrix multiplications and not on convolutions.
Our two main contributions are:
\begin{enumerate}
    \item \textbf{Trace norm regularization:} We describe a trace norm
    regularization technique and an accompanying training methodology that
    enables the practical training of models with competitive accuracy versus
    number of parameter trade-offs. It automatically selects the rank and
    eliminates the need for any prior knowledge on suitable matrix rank.
    \item \textbf{Efficient kernels for inference:} We explore the importance of
    optimizing for low batch sizes in on-device inference, and we introduce
    kernels\footnote{Available at \url{https://github.com/paddlepaddle/farm}.}
    for ARM processors that vastly outperform publicly available kernels in the
    low batch size regime.
\end{enumerate}

These two topics are discussed in Sections~\ref{sec:modelcomp}
and~\ref{sec:fastkernels}, respectively. Although we conducted our experiments
and report results in the context of large vocabulary continuous speech
recognition (LVCSR) on embedded devices, the ideas and techniques are broadly
applicable to other deep learning networks. Work on compressing any neural
network for which large GEMMs dominate the parameters or computation time could
benefit from the insights presented in this paper.

\section{Related work}
\label{sec:related}

Our work is most closely related to that of~\citet{prabhavalkar2016compression},
where low rank factored acoustic speech models are similarly trained by
initializing weights from a truncated singular value decomposition (SVD) of pretrained
weight matrices. This technique was also applied to speech recognition on mobile
devices~\citep{mcgraw2016personalized,xue2013restructuring}. We build on this
method by adding a variational form of trace norm regularization that was first
proposed for collaborative prediction~\citep{srebro2005maximum} and also applied
to recommender systems~\citep{koren2009matrix}. The use of this technique with
gradient descent was recently justified
theoretically~\citep{ciliberto2017reexamining}. Furthermore,
\citet{neyshabur2014search} argue that trace norm regularization could provide a
sensible inductive bias for neural networks. To the best of our knowledge, we are
the first to combine the training technique
of~\citet{prabhavalkar2016compression} with variational trace norm
regularization.

Low rank factorization of neural network weights in general has been the subject
of many other 
works~\citep{denil2013predicting,sainath2013low,ba2014deep,kuchaiev2017factorization}. 
Some other approaches for dense matrix compression include
sparsity~\citep{lecun1989optimal, narang2017exploring}, hash-based parameter
sharing~\citep{chen2015compressing}, and other parameter-sharing schemes such as
circulant, Toeplitz, or more generally low-displacement-rank
matrices~\citep{sindhwani2015structured,lu2016learning}.
\citet{kuchaiev2017factorization} explore splitting activations into
independent groups. Doing so is akin to using block-diagonal matrices.

The techniques for compressing convolutional models are different and beyond
the scope of this paper. We refer the interested reader to,
e.g.,~\citet{denton2014exploiting,han2015deep,iandola2016squeezenet} and
references therein.

\section{Training low rank models}
\label{sec:modelcomp}

Low rank factorization is a well studied and effective technique for compressing
large matrices. In~\citet{prabhavalkar2016compression}, low rank models are
trained by first training a model with unfactored weight matrices (we refer to this as
stage 1), and then initializing a model with factored weight matrices from the
truncated SVD of the unfactored model (we refer to this as \textit{warmstarting} a
stage 2 model from a stage 1 model). The truncation is
done by retaining only as many singular values as required to explain a
specified percentage of the variance.

If the weight matrices from stage 1 had only a few nonzero singular values, then the
truncated SVD used for warmstarting stage 2 would yield a much better or even error-free
approximation of the stage 1 matrix. This suggests
applying a sparsity-inducing $\ell^1$ penalty on the vector of singular values during
stage 1 training. This is known as trace norm regularization in the literature.
Unfortunately, there is no known way of directly computing
the trace norm and its gradients that would be computationally feasible in the
context of large deep learning models. Instead, we propose to combine the
two-stage training method of~\citet{prabhavalkar2016compression} with an indirect
variational trace norm regularization
technique~\citep{srebro2005maximum,ciliberto2017reexamining}. We describe this
technique in more detail in Section~\ref{sec:tracenormreg} and report
experimental results in Section~\ref{sec:compexpres}.

\subsection{Trace norm regularization}
\label{sec:tracenormreg}

First we introduce some notation. Let us denote by $||\cdot||_{\mathcal{T}}$ the
\textit{trace norm} of a matrix, that is, the sum of the singular values of the
matrix. The trace norm is also referred to as the \textit{nuclear norm} or the
\textit{Schatten 1-norm} in the literature. Furthermore, let us denote by
$||\cdot||_{\mathcal{F}}$ the Frobenius norm of a matrix, defined as
\begin{equation}
    || A ||_{\mathcal{F}} = \sqrt{\mathrm{Tr} A A^*} = \sqrt{\sum_{i, j} 
    |A_{ij}|^2} \,.
\end{equation}
The Frobenius norm is identical to the \textit{Schatten 2-norm} of a matrix, i.e.
the $\ell^2$ norm of the singular value vector of the matrix. The following lemma
provides a variational characterization of the trace norm in terms of the
Frobenius norm.

\begin{lemma}[\citet{jameson1987summing, ciliberto2017reexamining}]
Let $W$ be an $m \times n$ matrix and denote by $\sigma$ its vector of singular 
values. Then
\begin{equation}
\label{eq:varational-trace-norm}
  ||W||_{\mathcal{T}} := \sum_{i=1}^{\min(m, n)} \sigma_i(W)
  = \min \frac{1}{2} \left( ||U||_{\mathcal{F}}^2 + ||V||_{\mathcal{F}}^2 \right) \,,
\end{equation}
where the minimum is taken over all $U : m \times \min(m, n)$ and $V : \min(m, n)
\times n$ such that $W = UV$. Furthermore, if $W = \tilde{U} \Sigma \tilde{V}^*$
is a singular value decomposition of $W$, then equality holds
in~\eqref{eq:varational-trace-norm} for the choice $U = \tilde{U} \sqrt{\Sigma}$
and $V = \sqrt{\Sigma} \tilde{V}^*$.
\end{lemma}

The procedure to take advantage of this characterization is as follows.
First, for each large GEMM in the model, replace the $m \times n$ weight matrix $W$ by
the product $W = UV$ where $U : m \times \min(m, n)$ and $V : \min(m, n) \times
n$. Second, replace the original loss function $\ell(W)$ by
\begin{equation}\label{eq:modified-loss}
  \ell(U V) + \frac{1}{2} \lambda \left(||U||_{\mathcal{F}}^2 + 
  ||V||_{\mathcal{F}}^2 \right) \,.
\end{equation}
where $\lambda$ is a hyperparameter controlling the strength of the approximate
trace norm regularization. Proposition~1 in~\citet{ciliberto2017reexamining}
guarantees that minimizing the modified loss equation~\eqref{eq:modified-loss} is
equivalent to minimizing the actual trace norm regularized loss:
\begin{equation}
    \ell(W) + \lambda ||W||_{\mathcal{T}} \,.
\end{equation}

In Section~\ref{sec:stage1} we show empirically that use of the modified
loss~\eqref{eq:modified-loss} is indeed highly effective at reducing the trace
norm of the weight matrices.

To summarize, we propose the following basic training scheme:
\begin{itemize}
\item \textbf{Stage 1:}
    \begin{itemize}
        \item For each large GEMM in the model, replace the $m \times n$ weight 
        matrix $W$ by the product $W = UV$ where $U : m \times r$, $V : r \times n$, 
        and $r = \min(m, n)$.
        \item Replace the original loss function $\ell(W)$ by
            \begin{equation}
                \ell(U V) + \frac{1}{2} \lambda \left( ||U||_{\mathcal{F}}^2 +
                ||V||_{\mathcal{F}}^2 \right) \,,
            \end{equation}
            where $\lambda$ is a hyperparameter controlling the strength of the
            trace norm regularization.
        \item Train the model to convergence.
    \end{itemize}
\item \textbf{Stage 2:}
    \begin{itemize}
        \item For the trained model from stage 1, recover $W = UV$ by multiplying
        the two trained matrices $U$ and $V$.
        \item Train low rank models warmstarted from the truncated SVD of $W$. By
        varying the number of singular values retained, we can control the
        parameter versus accuracy trade-off.
    \end{itemize}
\end{itemize}

One modification to this is described in Section~\ref{sec:traintime}, where we
show that it is actually not necessary to train the stage 1 model to convergence
before switching to stage 2. By making the transition earlier, training time can
be substantially reduced.

\subsection{Experiments and results}
\label{sec:compexpres}

We report here the results of our experiments related to trace norm
regularization. Our baseline model is a forward-only Deep Speech
2 model, and we train
and evaluate on the widely used Wall Street Journal (WSJ) speech corpus.
Except for a few minor modifications described in Appendix~\ref{sec:othermethods},
we follow closely the original paper describing this 
architecture~\citep{amodei2016deep}, and we refer the reader to that paper for details
on the inputs, outputs, exact layers used, training methodology, and so on.
For the purposes of this paper, suffice it to say that the parameters and computation
are dominated by three GRU layers and a fully connected layer. It is these four layers
that we compress through low-rank factorization. As described in 
Appendix~\ref{sec:lrparamsharing}, in our factorization scheme, each GRU layer involves
two matrix multiplications: a \textit{recurrent} and a \textit{non-recurrent} one.
For a simple recurrent layer, we would write
\begin{equation}
    h_t = f(W_{nonrec} x_t + W_{rec} h_{t-1}) \,.
\end{equation}
For a GRU layer, there are also weights for reset and update gates, which we group
with the \textit{recurrent} matrix. See Appendix~\ref{sec:lrparamsharing} for details
and the motivation for this split.

Since our work focuses only on compressing acoustic models and not language
models, the error metric we report is the character error rate (CER) rather than
word error rate (WER). As the size and latency constraints vary widely across 
devices, whenever possible we compare techniques by comparing their accuracy versus 
number of parameter trade-off curves. All CERs reported here
are computed on a validation set separate from the training set.

\subsubsection{Stage 1 experiments}
\label{sec:stage1}

In this section, we investigate the effects of training with the modified loss
function in~\eqref{eq:modified-loss}. For simplicity, we refer to this as
\textit{trace norm regularization}.

As the WSJ corpus is relatively small at around 80 hours of speech, models tend
to benefit substantially from regularization. To make comparisons more fair, we
also trained unfactored models with an $\ell^2$ regularization term and searched
the hyperparameter space just as exhaustively.

For both trace norm and $\ell^2$ regularization, we found it beneficial to
introduce separate $\lambda_{rec}$ and $\lambda_{nonrec}$ parameters for
determining the strength of regularization for the recurrent and non-recurrent
weight matrices, respectively. In addition to $\lambda_{rec}$ and
$\lambda_{nonrec}$ in initial experiments, we also roughly tuned the learning
rate. Since the same learning rate was found to be optimal for nearly all
experiments, we just used that for all the experiments reported in this section.
The dependence of final CER on $\lambda_{rec}$ and $\lambda_{nonrec}$ is shown in
Figure~\ref{fig:cerheatmap}. Separate $\lambda_{rec}$ and $\lambda_{nonrec}$ values
are seen to help for both trace norm and $\ell^2$ regularization. However, for trace norm 
regularization, it appears better to fix $\lambda_{rec}$ as a multiple of 
$\lambda_{nonrec}$ rather than tuning the two parameters independently.

\begin{figure}
\centering
\includegraphics[scale=\figscale]{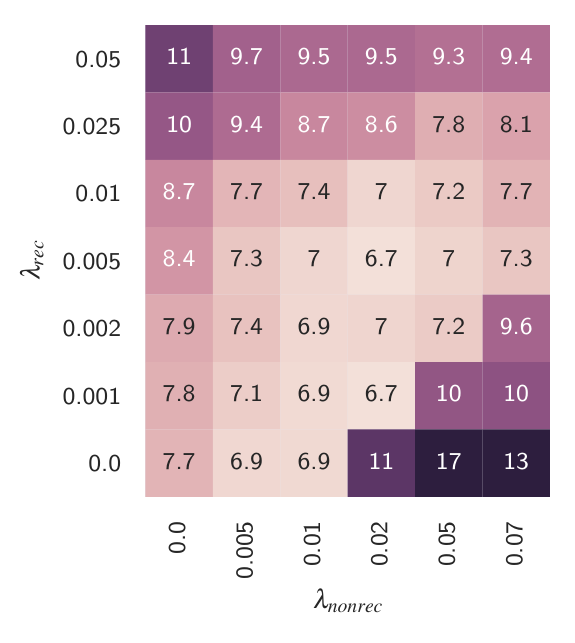}\quad
\includegraphics[scale=\figscale]{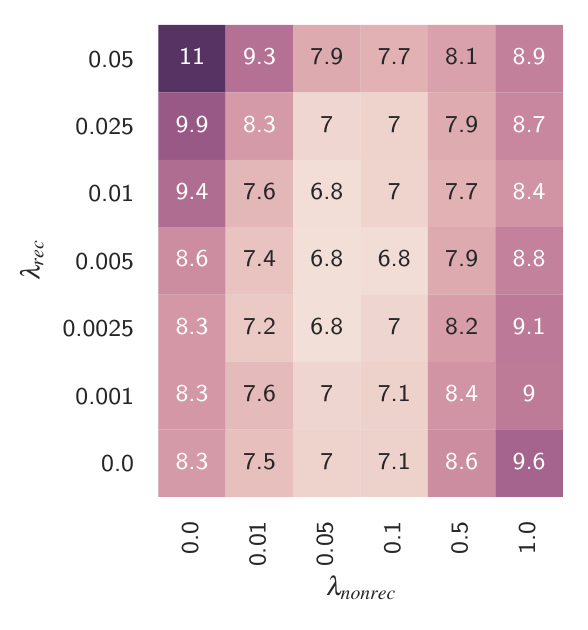}
\caption{CER dependence on $\lambda_{rec}$ and $\lambda_{nonrec}$ for trace norm
regularization (left) and $\ell^2$ regularization (right).}
\label{fig:cerheatmap}
\end{figure}

The first question we are interested in is whether our modified
loss~\eqref{eq:modified-loss} is really effective at reducing the trace norm. As
we are interested in the relative concentration of singular values rather than
their absolute magnitudes, we introduce the following nondimensional metric.

\begin{definition}\label{def:nondimTM}
Let $W$ be a nonzero $m \times n$ matrix with $d = \min(m, n) \geq 2$. Denote by
$\sigma$ the $d$-dimensional vector of singular values of $W$. Then we define the
\textit{nondimensional trace norm coefficient} of $W$ as follows:
\begin{equation}
\label{eq:nondim-TN}
    \nu(W) := \frac{\frac{||\sigma||_{\ell^1}}{||\sigma||_{\ell^2}} - 1}{\sqrt{d} - 1} \,.
\end{equation}
\end{definition}

We show in Appendix~\ref{sec:trcoeff} that $\nu$ is scale-invariant and ranges
from 0 for rank 1 matrices to 1 for maximal-rank matrices with all singular
values equal. Intuitively, the smaller $\nu(W)$, the better $W$ can be
approximated by a low rank matrix.

As shown in Figure~\ref{fig:l1byl2vslambda}, trace norm regularization is indeed
highly effective at reducing the nondimensional trace norm coefficient compared
to $\ell^2$ regularization. At very high regularization strengths, $\ell^2$
regularization also leads to small $\nu$ values. However, from 
Figure~\ref{fig:cerheatmap} it is apparent that this comes at the expense of
relatively high CERs. As shown in Figure~\ref{fig:rankvscer}, this translates
into requiring a much lower rank for the truncated SVD to explain, say, 90 \% of
the variance of the weight matrix for a given CER. Although a few
$\ell^2$-regularized models occasionally achieve low rank, we observe this only
at relatively high CER's and only for some of the weights. Note also that
some form of regularization is very important on this dataset. The unregularized
baseline model (the green points in Figure 3) achieves relatively low CER.

\begin{figure}[ptb]
\centering
\includegraphics[scale=\figscale]{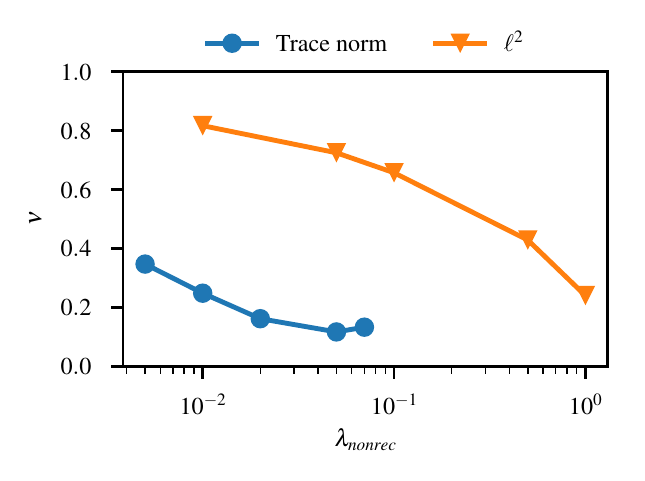}\quad
\includegraphics[scale=\figscale]{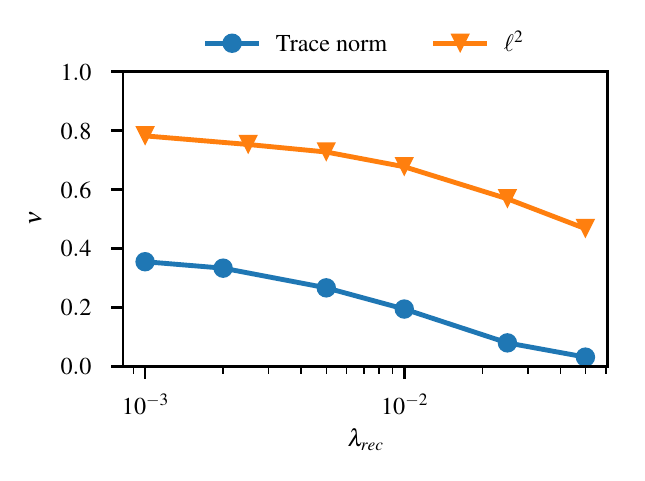}
\caption{Nondimensional trace norm coefficient versus strength of regularization
by type of regularization used during training. On the left are the results for
the non-recurrent weight of the third GRU layer, with $\lambda_{rec} = 0$. On the
right are the results for the recurrent weight of the third GRU layer, with
$\lambda_{nonrec} = 0$.  The plots for the other weights are similar.}
\label{fig:l1byl2vslambda}
\end{figure}

\begin{figure}[ptb]
\centering
\includegraphics[scale=\figscale]{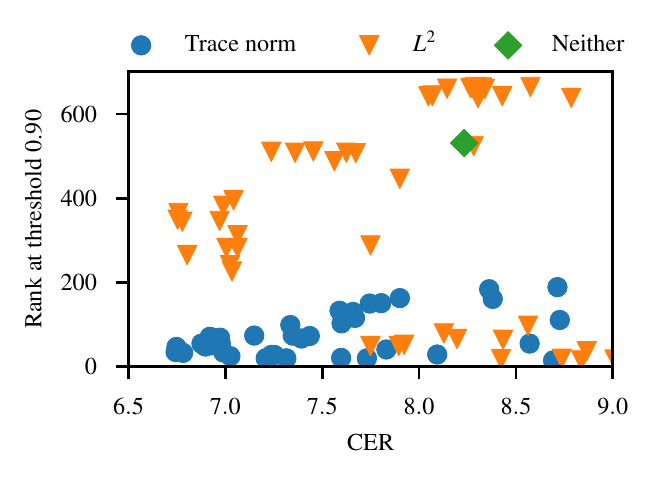}\quad
\includegraphics[scale=\figscale]{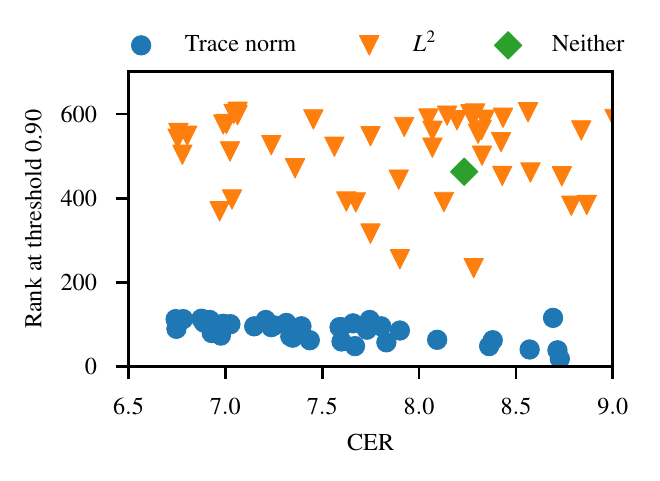}
\caption{The truncated SVD rank required to explain 90 \% of the variance of the
weight matrix versus CER by type of regularization used during training. Shown
here are results for the non-recurrent (left) and recurrent (right) weights of
the third GRU layer. The plots for the other weights are similar.}
\label{fig:rankvscer}
\end{figure}

\begin{figure}[ptb]
\centering
\includegraphics[scale=\figscale]{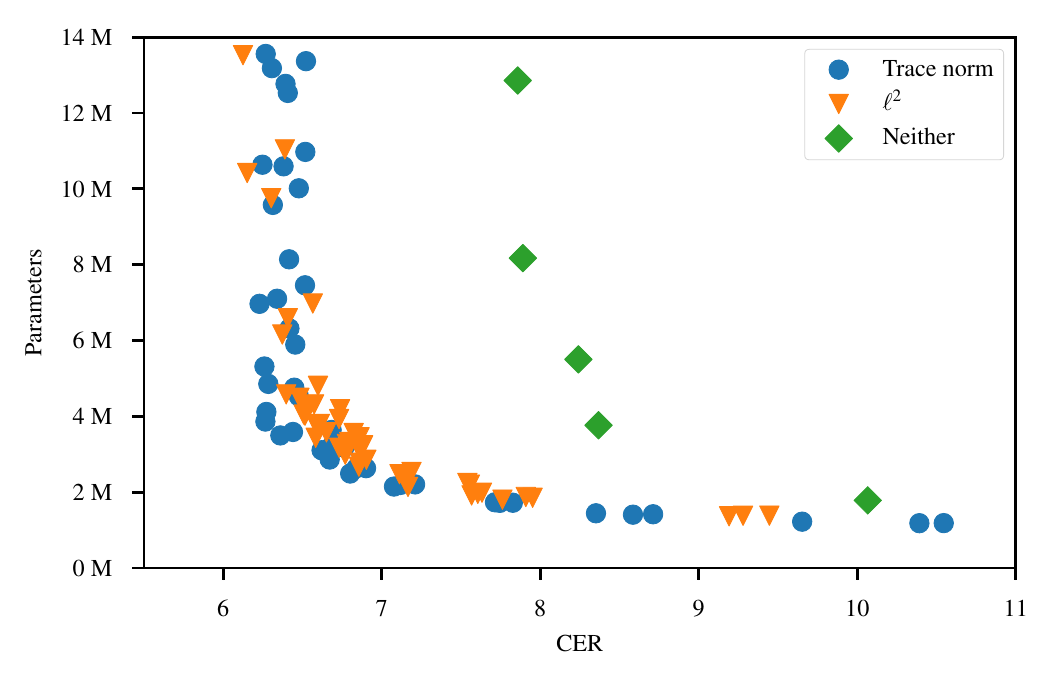}
\caption{Number of parameters versus CER of stage 2 models colored by the type of
regularization used for training the stage 1 model.}
\label{fig:paramvscer}
\end{figure}

\subsubsection{Stage 2 experiments}
\label{sec:stage2}

In this section, we report the results of stage 2 experiments warmstarted from
either trace norm or $L^2$ regularized stage 1 models.

For each regularization type, we took the three best stage 1 models (in terms of
final CER: all were below 6.8) and used the truncated SVD of their weights to
initialize the weights of
stage 2 models. By varying the threshold of variance explained for the SVD
truncation, each stage 1 model resulted into multiple stage 2 models. The stage 2
models were trained without regularization (i.e., $\lambda_{rec} =
\lambda_{nonrec} = 0$) and with the initial learning rate set to three times the
final learning rate of the stage 1 model. 

As shown in Figure~\ref{fig:paramvscer}, the best models from either trace norm
or $L^2$ regularization exhibit similar accuracy versus number of parameter
trade-offs. For comparison, we also warmstarted some stage 2 models from an
unregularized stage 1 model. These models are seen to have significantly lower
accuracies, accentuating the need for regularization on the WSJ corpus.

\subsubsection{Reducing training time}
\label{sec:traintime}

In the previous sections, we trained the stage 1 models for 40 epochs to full
convergence and then trained the stage 2 models for another 40 epochs, again to
full convergence. Since the stage 2 models are drastically smaller than the stage
1 models, it takes less time to train them. Hence, shifting the stage 1 to stage
2 transition point to an earlier epoch could substantially reduce training time.
In this section, we show that it is indeed possible to do so without hurting
final accuracy.

Specifically, we took the stage 1 trace norm and $\ell^2$ models from
Section~\ref{sec:stage1} that resulted in the best stage 2 models in
Section~\ref{sec:stage2}. In that section, we were interested in the parameters
vs accuracy trade-off and used each stage 1 model to warmstart a number of stage
2 models of different sizes. In this section, we instead set a fixed target of 3
M parameters and a fixed overall training budget of 80 epochs but vary the stage
1 to stage 2 transition epoch. For each of the stage 2 runs, we initialize the
learning rate with the learning rate of the stage 1 model at the transition
epoch. So the learning rate follows the same schedule as if we had trained a
single model for 80 epochs. As before, we disable all regularization for stage 2.

The $\ell^2$ stage 1 model has 21.7 M parameters, whereas the trace norm stage 1
model at 29.8 M parameters is slightly larger due to the factorization. Since the
stage 2 models have roughly 3 M parameters and the training time is approximately
proportional to the number of parameters, stage 2 models train about 7x and 10x
faster, respectively, than the $\ell^2$ and trace norm stage 1 models.
Consequently, large overall training time reductions can be achieved by reducing
the number of epochs spent in stage 1 for both $\ell^2$ and trace norm.

The results are shown in Figure~\ref{fig:traintime}.  Based on the left panel, it
is evident that we can lower the transition epoch number without hurting the
final CER. In some cases, we even see marginal CER improvements. For transition
epochs of at least 15, we also see slightly better results for trace norm than
$\ell^2$. In the right panel, we plot the convergence of CER when the transition
epoch is 15. We find that the trace norm model's CER is barely impacted by the
transition whereas the $\ell^2$ models see a huge jump in CER at the transition
epoch. Furthermore, the plot suggests that a total of 60 epochs may have
sufficed. However, the savings from reducing stage 2 epochs are negligible
compared to the savings from reducing the transition epoch.

\begin{figure}
\centering
\includegraphics[scale=\figscale]{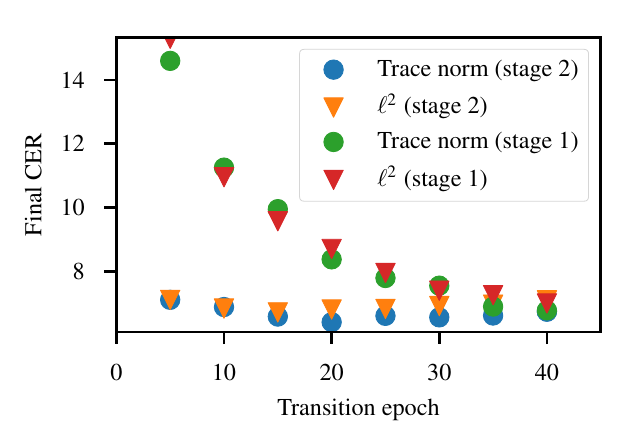}\quad
\includegraphics[scale=\figscale]{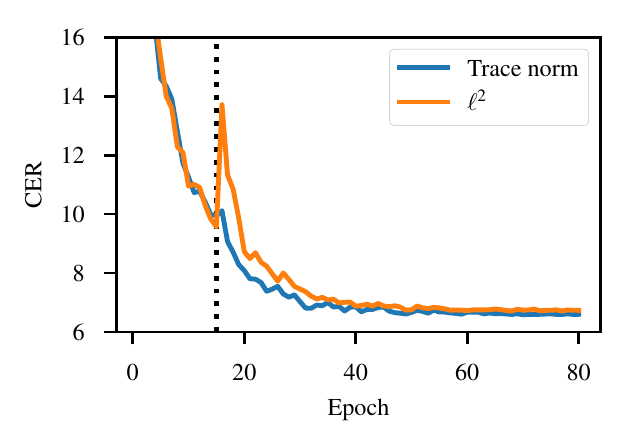}
\caption{\textit{Left:} CER versus transition epoch, colored by the type of 
regularization used for training the stage 1 model. \textit{Right:} CER as 
training progresses colored by the type of regularization used in stage 1. 
The dotted line indicates the transition epoch.}
\label{fig:traintime}
\end{figure}

\section{Application to production-grade embedded speech recognition}
\label{sec:fastkernels}

With low rank factorization techniques similar\footnote{This work was done prior
to the development of our trace norm regularization. Due to long training cycles
for the 10,000+ hours of speech used in this section, we started from pretrained
models. However, the techniques in this section are entirely agnostic to such
differences.} to those described in Section~\ref{sec:modelcomp}, we were able to
train large vocabulary continuous speech recognition (LVCSR) models with
acceptable numbers of parameters and acceptable loss of accuracy compared to a
production server model (baseline). Table~\ref{tab:finalam} shows the baseline
along with three different compressed models with much lower number of
parameters. The tier-3 model employs the techniques of Sections~\ref{sec:gramctc}
and~\ref{sec:mel}. Consequently, it runs significantly faster than the tier-1
model, even though they have a similar number of parameters. Unfortunately, this
comes at the expense of some loss in accuracy.
 
\begin{table}[t]
\caption{WER of three tiers of low rank speech recognition models and a
production server model on an internal test set. This table illustrates the
effect of shrinking just the acoustic model. The same large server-grade language
model was used for all rows.}
\label{tab:finalam}
\begin{center}
\begin{tabular}{lSSS[table-format=2.1]}
    \toprule
    {Model}   & {Parameters (M)} & {WER} & {\% Relative}      \\ \midrule
    baseline  & 115.5            & 8.78  & 0.0\si{\percent}   \\
    tier-1    & 14.9             & 9.25  & -5.4\si{\percent}  \\
    tier-2    & 10.9             & 9.80  & -11.6\si{\percent} \\
    tier-3*    & 14.7            & 9.92  & -13.0\si{\percent} \\
    \bottomrule
\multicolumn{4}{l}{\textit{{*} The tier-3 model is larger but faster
than the tier-2 model. See main text for details.}}
\end{tabular}
\end{center}
\end{table}

Although low rank factorization significantly reduces the overall computational
complexity of our LVCSR system, we still require further optimization to achieve
real-time inference on mobile or embedded devices. One approach to speeding up
the network is to use low-precision $8$-bit integer representations for weight
matrices and matrix multiplications (the GEMM operation in BLAS terminology).
This type of quantization after training reduces both memory as well as
computation requirements
of the network while only introducing $2\%$ to $4\%$ relative increase in WER.
Quantization for embedded speech recognition has also been previously 
studied in~\citep{alvarez2016efficient,vanhoucke2011improving}, and it may be possible
to reduce the relative WER increase by quantizing the forward passes during
training~\citep{alvarez2016efficient}. As the relative WER losses from compressing
the acoustic and language models were much larger for us, we did not pursue this
further for the present study.

To perform low precision matrix multiplications, we originally used the
\textit{gemmlowp} library, which provides state-of-the-art low precision GEMMs
using unsigned 8-bit integer values~\citep{gemmlowp}.  However, gemmlowp's
approach is not efficient for small batch  sizes. Our application, LVCSR on
embedded devices with single user, is dominated by low batch size GEMMs due to
the sequential nature of recurrent layers and latency constraints. This can be
demonstrated by looking at a simple RNN cell which has the form:
\begin{equation}
    h_t = f(Wx_t + Uh_{t-1})
    \label{eq:simpleRNN}
\end{equation}
This cell contains two main GEMMs: The first, $Uh_{t-1}$, is sequential and
requires a GEMM with batch size 1. The second, $Wx_t$, can in principle be
performed at higher batch sizes by batching across time. However, choosing a too
large batch sizes can significantly delay the output, as the system needs to wait
for more future context. In practice, we found that batch sizes higher than
around 4 resulted in too high latencies, negatively impacting user experience.

This motivated us to implement custom assembly kernels for the 64-bit ARM
architecture (AArch64, also known as ARMv8 or ARM64) to further improve the
performance of the GEMMs operations. We do not go through the methodological
details in this paper. Instead, we are making the kernels and implementation
details available at \url{https://github.com/paddlepaddle/farm}. 

Figure~\ref{fig:kernel_perf_bw} compares the performance of our implementation
(denoted by {\it farm}) with the gemmlowp library for matrix multiplication on 
iPhone 7, iPhone 6, and Raspberry Pi 3 Model B. The farm kernels are
significantly  faster than their gemmlowp counterparts for batch sizes 1 to 4.
The peak single-core theoretical performance for iPhone 7, iPhone 6, and
Raspberry Pi 3 are $56.16$, $22.4$ and $9.6$ Giga Operations per Second, 
respectively. The gap between the theoretical and achieved values are mostly due
to kernels being limited by memory bandwidth. For a more detailed analysis, we
refer to the farm website.

\begin{figure}
\centering
{
\includegraphics[scale=\figscale]{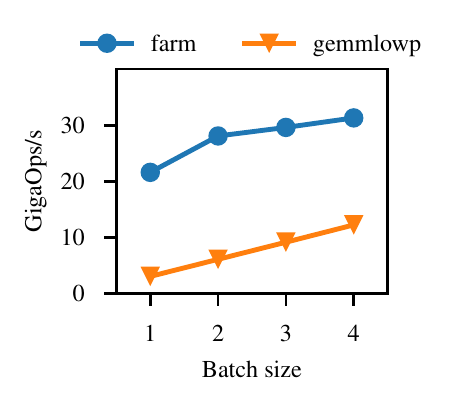}
\includegraphics[scale=\figscale]{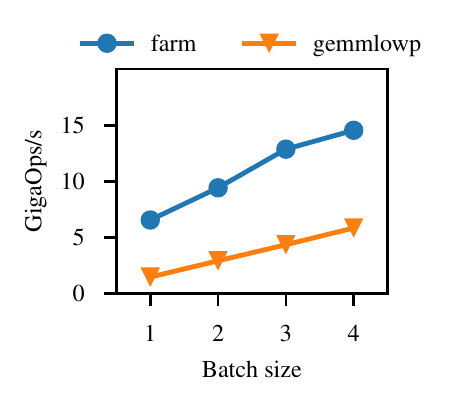}
\includegraphics[scale=\figscale]{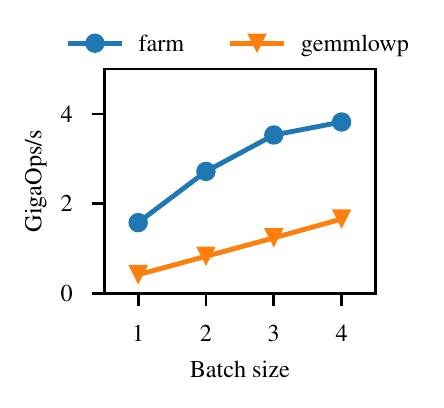}
}
\caption{Comparison of our kernels ({\it farm}) and the gemmlowp library for
matrix multiplication on iPhone 7 (left), iPhone 6 (middle), and Raspberry Pi 3
(right). The benchmark computes $Ax = b$ where $A$ is a random matrix with
dimension $6144 \times 320$, and $x$ is a random matrix with dimension $320 
\times$ batch size. All matrices are in unsigned 8-bit integer format.}
\label{fig:kernel_perf_bw}
\end{figure}

In addition to low precision representation and customized ARM kernels, we
explored other approaches to speed up our LVCSR system. These techniques are
described in Appendix~\ref{sec:othermethods}.

Finally, by combining low rank factorization, some techniques from 
Appendix~\ref{sec:othermethods}, int8 quantization and the farm kernels, as well
as using smaller language models, we could create a range of speech recognition
models suitably tailored to various devices. These are shown in
Table~\ref{tab:finalmodels}.

\begin{table}[t]
\caption{Embedded speech recognition models.}
\label{tab:finalmodels}
\begin{center}
\begin{tabular}{llS[table-format=5.0]SS[table-format=2.1]SS[table-format=2.1]}
    \toprule
              &            & {Language} &    &               &                
              & {\% time spent} \\
              & {Acoustic} & {model}    &    &               & {Speedup over} 
              & {in acoustic}\\
     {Device} & {model}    & {size (MB)}     & {WER} & {\% Relative} & {real-time} 
     & {model} \\
    \midrule
    GPU server     & baseline       & 13764      & 8.78  & 0.0\si{\percent}     
    & 10.39x     & 70.8\si{\percent}            \\
    iPhone 7       & tier-1         & 56         & 10.50 & -19.6\si{\percent}   
    & 2.21x          & 65.2\si{\percent}     \\
    iPhone 6       & tier-2         & 32         & 11.19 & -27.4\si{\percent}    
    & 1.13x          & 75.5\si{\percent}     \\
    Raspberry Pi 3 & tier-3         & 14         & 12.08     &  -37.6\si{\percent} 
    & 1.08x       & 86.3\si{\percent}                  \\
\bottomrule
\end{tabular}
\end{center}
\end{table}

\section{Conclusion}
\label{sec:conclusion}

We worked on compressing and reducing the inference latency of LVCSR speech
recognition models. To better compress models, we introduced a trace norm
regularization technique and demonstrated its potential for faster
training of low rank models on the WSJ speech corpus. To reduce
latency at inference time, we demonstrated the importance of optimizing for low
batch sizes and released optimized kernels for the ARM64 platform. Finally, by
combining the various techniques in this paper, we demonstrated an effective path
towards production-grade on-device speech recognition on a range of embedded
devices.

\subsubsection*{Acknowledgments}

We would like to thank Gregory Diamos, Christopher Fougner, Atul Kumar,
Julia Li, Sharan Narang, Thuan Nguyen, Sanjeev Satheesh, Richard Wang,
Yi Wang, and Zhenyao Zhu for their helpful comments and assistance with
various parts of this paper. We also thank anonymous referees for their
comments that greatly improved the exposition and helped uncover a mistake
in an earlier version of this paper.

\bibliography{ondevice_asr}
\bibliographystyle{iclr2018_conference}

\appendix

\section{Nondimensional trace norm coefficient}
\label{sec:trcoeff}

In this section, we describe some of the properties of the non-dimensional trace
norm coefficient defined in Section~\ref{sec:tracenormreg}.

\begin{proposition}\label{prop:nondimTM}
Let $W, d, \sigma$ be as in Definition~\ref{def:nondimTM}. Then
\begin{enumerate}[(i)]
    \item\label{trc1} $\nu(c W) = \nu(W)$ for all scalars $c \in \mathbb{R} 
    \setminus \{ 0 \}$.
    \item\label{trc2} $0 \leq \nu(W) \leq 1$.
    \item\label{trc3} $\nu(W) = 0$ if and only if $W$ has rank 1.
    \item\label{trc4} $\nu(W) = 1$ if and only if $W$ has maximal rank and
    all singular values are equal.
\end{enumerate}
\end{proposition}
\begin{proof}
Since we are assuming $W$ is nonzero, at least one singular value is nonzero
and hence $||\sigma||_{\ell^2} \neq 0$. Property~\eqref{trc1} is immediate 
from the scaling property $|| c \sigma|| = |c| \cdot ||\sigma||$ satisfied 
by all norms.

To establish the other properties, observe that we have
\begin{equation}
    (\sigma_i + \sigma_j)^2 \geq \sigma_i^2 + \sigma_j^2 \geq 2 \left( 
    \frac12 \sigma_i + \frac12 \sigma_j \right)^2 \,.
\end{equation}
The first inequality holds since singular values are nonnegative, and the 
inequality is strict unless $\sigma_i$ or $\sigma_j$ vanishes. The second 
inequality comes from an application of Jensen's inequality and is strict 
unless $\sigma_i = \sigma_j$. Thus, replacing $(\sigma_i, \sigma_j)$ by 
$(\sigma_i + \sigma_j, 0)$ preserves $||\sigma||_{\ell^1}$ while increasing 
$||\sigma||_{\ell^2}$ unless one of $\sigma_i$ or $\sigma_j$ is zero. 
Similarly, replacing $(\sigma_i, \sigma_j)$ by $(\frac12 \sigma_i + 
\frac12 \sigma_j, \frac12 \sigma_i + \frac12 \sigma_j)$ preserves 
$||\sigma||_{\ell^1}$ while decreasing $||\sigma||_{\ell^2}$ unless 
$\sigma_i = \sigma_j$. By a simple argument by contradiction, it follows 
that the minima occur for $\sigma = (\sigma_1, 0, \dotsc, 0)$, in which 
case $\nu(W) = 0$ and the maxima occur for
$\sigma = (\sigma_1, \dotsc, \sigma_1)$, in which case $\nu(W) = 1$. 
\end{proof}

We can also obtain a better intuition about the minimum and maximum of $\nu(W)$
by looking at the 2D case visualized in Figure~\ref{fig:tracenormcoeffsketch}.
For a fixed $||\sigma||_{\ell^2}=\sigma$, $||\sigma||_{\ell^1}$ can vary from 
$\sigma$ to  $\sqrt2\sigma$. The minimum $||\sigma||_{\ell^1}$ happens when either 
$\sigma_1$ or $\sigma_2$ are zero. For these values  
$||\sigma||_{\ell^2}=||\sigma||_{\ell^1}$ and as a result $\nu(W)=0$. 
Similarly, the maximum $||\sigma||_{\ell^1}$ happens for $\sigma_1=\sigma_2$, 
resulting in $\nu(W)=1$.

\begin{figure}
\centering
\includegraphics[scale=0.15]{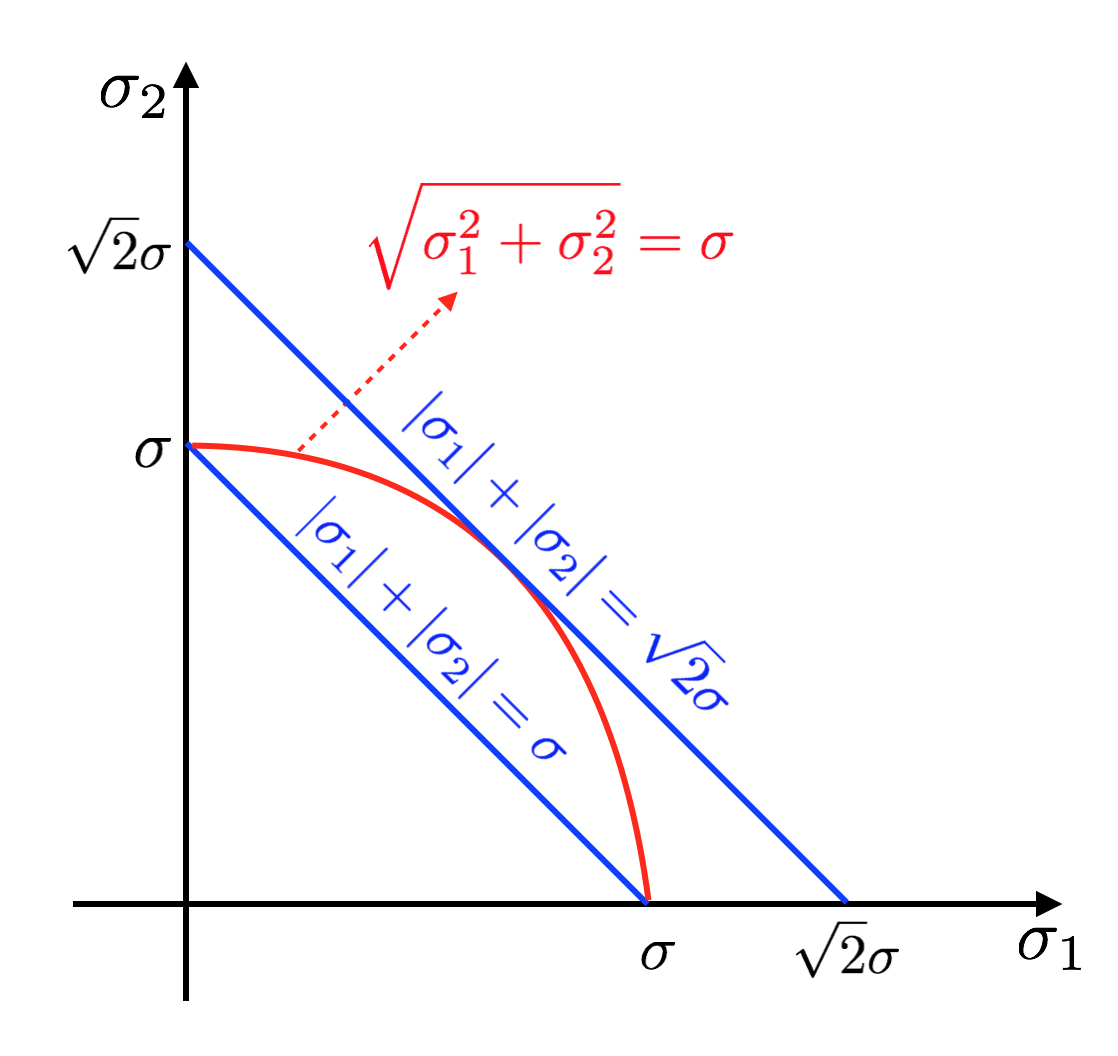}
\caption{Contours of $||\sigma||_{\ell^1}$ and $||\sigma||_{\ell^2}$. 
$||\sigma||_{\ell^2}$ is kept constant at $\sigma$. For this case, 
$||\sigma||_{\ell^1}$ can vary from $\sigma$ to $\sqrt2\sigma$.}
\label{fig:tracenormcoeffsketch}
\end{figure}

\section{Model design considerations}
\label{sec:othermethods}

We describe here a few preliminary insights that informed our choice of baseline
model for the experiments reported in Sections~\ref{sec:modelcomp}
and~\ref{sec:fastkernels}.

Since the target domain is on-device streaming speech recognition with low
latency, we chose to focus on Deep Speech 2 like models with forward-only GRU
layers~\citep{amodei2016deep}.

\subsection{Growing recurrent layer sizes}
\label{sec:growrec}

Across several data sets and model architectures, we consistently found that the
sizes of the recurrent layers closer to the input could be shrunk without
affecting accuracy much. A related phenomenon was observed
in~\citet{prabhavalkar2016compression}: When doing low rank approximations of the
acoustic model layers using SVD, the rank required to explain a fixed threshold
of explained variance grows with distance from the input layer.

To reduce the number of parameters of the baseline model and speed up
experiments, we thus chose to adopt growing GRU dimensions. Since the hope is
that the compression techniques studied in this paper will automatically reduce
layers to a near-optimal size, we chose to not tune these dimensions, but simply
picked a reasonable affine increasing scheme of 768, 1024, 1280 for the GRU
dimensions, and dimension 1536 for the final fully connected layer.

\subsection{Parameter sharing in the low rank factorization}
\label{sec:lrparamsharing}

For the recurrent layers, we employ the Gated Recurrent Unit (GRU) architecture
proposed in~\citet{cho2014properties,chung2014empirical}, where the hidden state
$h_t$ is computed as follows:
\begin{equation}
  \label{eq:t}
  \begin{aligned}
    z_t &= \sigma(W_z x_t + U_z h_{t - 1} + b_z)\\        
    r_t &= \sigma(W_r x_t + U_r h_{t -1} + b_r )\\
    \tilde{h}_t &= f(W_h x_t + r_t \cdot U_h h_{t-1} + b_h) \\
    h_t &= (1 - z_t) \cdot h_{t - 1} + z_t \cdot \tilde{h}_t
  \end{aligned}
\end{equation}
where $\sigma$ is the sigmoid function, $z$ and $r$ are update and reset gates
respectively, $U_z, U_r, U_h$ are the three recurrent weight matrices, and $W_z,
W_r, W_h$ are the three non-recurrent weight matrices.

We consider here three ways of performing weight sharing when doing low rank
factorization of the 6 weight matrices.

\begin{enumerate}
\item \textbf{Completely joint factorization.} Here we concatenate the 6 weight
matrices along the first dimension and apply low rank factorization to this
single combined matrix. 
\item \textbf{Partially joint factorization.} Here we concatenate the 3 recurrent
matrices into a single matrix $U$ and likewise concatenate the 3 non-recurrent
matrices into a single matrix $W$. We then apply low rank factorization to each
of $U$ and $W$ separately.
\item \textbf{Completely split factorization.} Here we apply low rank
factorization to each of the 6 weight matrices separately.
\end{enumerate}

In~\citep{prabhavalkar2016compression,kuchaiev2017factorization}, the authors
opted for the LSTM analog of \textit{completely joint factorization}, as this
choice has the most parameter sharing and thus the highest potential for
compression of the model. However, we decided to go with \textit{partially joint
factorization} instead, largely for two reasons. First, in pilot experiments, we
found that the $U$ and $W$ matrices behave qualitatively quite differently during
training. For example, on large data sets the $W$ matrices may be trained from
scratch in factored form, whereas factored $U$ matrices need to be either
warmstarted via SVD from a trained unfactored model or trained with a
significantly lowered learning rate. Second, the $U$ and $W$ split is
advantageous in terms of computational efficiency. For the non-recurrent $W$
GEMM, there is no sequential time dependency and thus its inputs $x$ may be
batched across time.

Finally, we compared the partially joint factorization to the completely split
factorization and found that the former indeed led to better accuracy versus
number of parameters trade-offs. Some results from this experiment are shown in
Table~\ref{tab:split_shared_wsj}.

\begin{table}[t]
\caption{Performance of completely split versus partially joint factorization of
recurrent weights.}
\label{tab:split_shared_wsj}
\begin{center}
\begin{tabular}{SSSSS}
\toprule
  & \multicolumn{2}{c}{Completely split} & \multicolumn{2}{c}{Partially joint} \\
\cmidrule(r){2-3} \cmidrule(r){4-5}
         {SVD threshold} & {Parameters (M)}  &  {CER} & {Parameters (M)} &   {CER} \\
\midrule
       0.50 & 6.3  & 10.3 & 5.5 &  10.3 \\
       0.60 & 8.7  & 10.5 & 7.5 &  10.2 \\
        0.70 & 12.0  & 10.3 & 10.2 & 9.9 \\
        0.80 & 16.4 & 10.1 & 13.7 &  9.7 \\
\bottomrule
\end{tabular}
\end{center}
\end{table}

\subsection{Mel and smaller convolution filters}\label{sec:mel}

Switching from 161-dimensional linear spectrograms to 80-dimensional mel
spectrograms reduces the per-timestep feature dimension by roughly a factor of 2.
Furthermore, and likely owing to this switch, we could reduce the
frequency-dimension size of the convolution filters by a factor of 2. In
combination, this means about a 4x reduction in compute for the first and second
convolution layers, and a 2x reduction in compute for the first GRU layer.

On the WSJ corpus as well as an internal dataset of around 1,000 hours of speech,
we saw little impact on accuracy from making this change, and hence we adopted it
for all experiments in Section~\ref{sec:modelcomp}.

\subsection{Gram-CTC and increased stride in convolutions}\label{sec:gramctc}

Gram-CTC is a recently proposed extension to CTC for training models that output
variable-size grams as opposed to single characters~\citep{liu2017gram}. Using
Gram-CTC, we were able to increase the time stride in the second convolution
layer by a factor of 2 with little to no loss in CER, though we did have to
double the number of filters in that same convolution layer to compensate. The
net effect is a roughly 2x speedup for the second and third GRU layers, which
are the largest. This speed up more than makes up for the size increase in the
softmax layer and the slightly more complex language model decoding when using
Gram-CTC. However, for a given target accuracy, we found that Gram-CTC models
could not be shrunk as much as CTC models by means of low rank factorization.
That is, the net effect of this technique is to increase model size in exchange for
reduced latency.

\subsection{low rank factorization versus learned sparsity}
\label{sec:lrvssparse}

Shown in Figure~\ref{fig:10p_cer_vs_params} is the parameter reduction versus
relative CER increase trade-off for various techniques on an internal data set of
around 1,000 hours of speech.

\begin{figure}
\centering
\includegraphics[scale=\figscale]{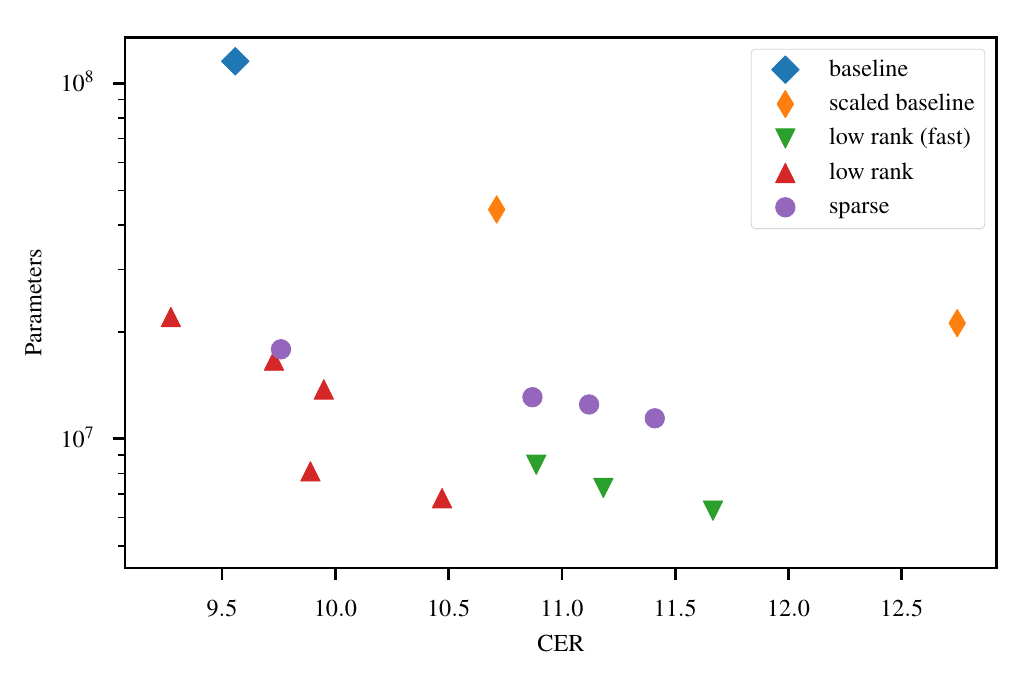}
\caption{CER versus parameter on an internal dataset, colored by parameter
reduction technique.}
\label{fig:10p_cer_vs_params}
\end{figure}

The baseline model is a Deep Speech 2 model with three forward-GRU layers of
dimension 2560, as described in~\citet{amodei2016deep}. This is the same baseline
model used in the experiments of~\citet{narang2017exploring}, from which paper we
also obtained the sparse data points in the plot. Shown also are versions of the
baseline model but with the GRU dimension scaled down to 1536 and 1024. Overall,
models with low rank factorizations on all non-recurrent and recurrent weight
matrices are seen to provide the best CER vs parameters trade-off. All the low
rank models use growing GRU dimensions and the partially split form of low rank
factorization, as discussed in Sections~\ref{sec:growrec}
and~\ref{sec:lrparamsharing}. The models labeled \textit{fast} in addition use
Gram-CTC as described in Section~\ref{sec:gramctc} and mel features and reduced
convolution filter sizes as described in Section~\ref{sec:mel}.

As this was more of a preliminary comparison to some past experiments, the setup
was not perfectly controlled and some models were, for example, trained for more
epochs than others. We suspect that, given more effort and similar adjustments
like growing GRU dimensions, the sparse models could be made competitive with the
low rank models. Even so, given the computational advantage of the low rank
approach over unstructured sparsity, we chose to focus only on the former going
forward. This does not, of course, rule out the potential usefulness of other,
more structured forms of sparsity in the embedded setting.

\end{document}